\documentclass{article}

\usepackage[english]{babel} 
\usepackage{microtype}      
\usepackage{csquotes}       
\usepackage[normalem]{ulem} 
\usepackage{xspace}         
\usepackage{textcomp}       

\usepackage{type1cm}        
\usepackage{bold-extra}     
\usepackage{bm}             
\usepackage{bbm}            
\usepackage{bbold}          
\usepackage{siunitx}        
\usepackage{xfrac}          

\usepackage{amsmath,amsfonts,amssymb,amsthm,mathtools} 
\usepackage{physics}        
\usepackage{thmtools,thm-restate} 

\usepackage{booktabs}       
\usepackage{multirow}       
\usepackage{array}          

\usepackage{graphicx}       
\usepackage[style=base, tableposition=top]{caption} 
\usepackage[caption=false,font=normalsize,labelfont=sf,textfont=sf]{subfig} 
\usepackage{tikz}           
\usetikzlibrary{patterns, positioning, arrows, bayesnet, calc} 
\usepackage{tikz-cd}        
\usepackage{pgfplots}       
\usepgfplotslibrary{patchplots}
\pgfplotsset{compat=1.15}   
\usepackage{tcolorbox}      

\usepackage{algorithmic}

\usepackage[bookmarks,colorlinks=true,pagebackref=true,backref=page]{hyperref} 
\usepackage[capitalize]{cleveref} 
\usepackage{cite}

\usepackage{balance}        
\usepackage{hyphenat}       
\usepackage[show]{chato-notes} 
\usepackage{stfloats}       
\usepackage{verbatim}       

\newcommand{\spara}[1]{\smallskip\noindent\textbf{#1}}

\newcommand{\epara}[1]{\smallskip\noindent\emph{#1}}

{\par\egroup\vskip 0.25ex}

\renewcommand*\backref[1]{\ifx#1\relax \else (Cited on #1) \fi}

\usepackage[dvipsnames]{xcolor}
\hypersetup{
   colorlinks=true,
   linkcolor={red!50!black},
   filecolor={green!50!black},
   citecolor={green!50!black}, 
   urlcolor={blue!80!black},
}

\definecolor{mypurple}{RGB}{254, 68, 218}
\definecolor{myred}{HTML}{E13D66}
\definecolor{mycyan}{HTML}{70D7D0}
\definecolor{mylightblue}{HTML}{2274A5}
\definecolor{mydarkblue}{HTML}{0C0A3E}

\theoremstyle{plain}
\newtheorem{theorem}{Theorem}[section]

\newtheorem{lemma}[theorem]{Lemma}
\newtheorem{corollary}[theorem]{Corollary}

\theoremstyle{definition}
\newtheorem{definition}[theorem]{Definition}

\theoremstyle{remark}

\crefname{theorem}{Thm.}{Thms.}
\crefname{proposition}{Prop.}{Props.}
\crefname{lemma}{lem.}{lems.}
\crefname{corollary}{Cor.}{Cors.}
\crefname{definition}{Def.}{Defs.}
\crefname{section}{Sec.}{Secs.}
\crefname{figure}{Fig.}{Figs.}
\crefname{problem}{Prob.}{Probs.}
\crefname{appendix}{App.}{Apps.}
\crefname{equation}{Eq.}{Eqs.}

\DeclareMathOperator*{\argmin}{arg\,min}

\newcommand{\frob}[1]{\ensuremath{\norm{#1}_{\mathrm{F}}}\xspace}
\newcommand{\KL}[3]{\ensuremath{D^{\mathrm{KL}}_{#1}\left(#2 || #3\right)}\xspace}

\newcommand{\measure}{\ensuremath{\chi}\xspace}

\newcommand{\prox}{\ensuremath{\mathrm{prox}}\xspace}

\newcommand{\spectral}{\textsc{spectral}\xspace}
\newcommand{\ntrials}{\textsf{\small{ntrials}}\xspace}

\newcommand{\pd}{\ensuremath{\mathcal{S}_{++}}\xspace}

\newcommand{\reall}{\ensuremath{\mathbb{R}}\xspace}
\newcommand{\scmcollection}{\ensuremath{{\mathcal{K}}}\xspace}

\newcommand{\edgeset}{\ensuremath{\mathcal{E}}\xspace}

\newcommand{\stiefel}[2]{\ensuremath{\mathrm{St}({#1},{#2})}\xspace}

\newcommand{\zeros}{\ensuremath{\boldsymbol{0}}\xspace}

\newcommand{\A}{\ensuremath{\mathbf{A}}\xspace}
\newcommand{\B}{\ensuremath{\mathbf{B}}\xspace}
\newcommand{\C}{\ensuremath{\mathbf{C}}\xspace}
\newcommand{\D}{\ensuremath{\mathbf{D}}\xspace}

\newcommand{\identity}{\ensuremath{\mathbf{I}}\xspace}
\newcommand{\K}{\ensuremath{\mathbf{K}}\xspace}

\newcommand{\T}{\ensuremath{\mathbf{T}}\xspace}
\newcommand{\U}{\ensuremath{\mathbf{U}}\xspace}
\newcommand{\V}{\ensuremath{\mathbf{V}}\xspace}
\newcommand{\Vhat}{\ensuremath{\widehat{\mathbf{V}}}\xspace}

\newcommand{\Y}{\ensuremath{\mathbf{Y}}\xspace}

\newcommand{\BPsi}{\ensuremath{\boldsymbol{\Psi}}\xspace}
\newcommand{\BUpsilon}{\ensuremath{\boldsymbol{\Upsilon}}\xspace}

\newcommand{\myendogenous}{\ensuremath{\mathcal{X}}\xspace}

\newcommand{\scm}[1]{\ensuremath{\mathsf{M}^{#1}}\xspace}

\newcommand{\abst}{\ensuremath{\boldsymbol{\alpha}}\xspace}

\newcommand{\covhigh}{\ensuremath{\boldsymbol{\Sigma}^{h}}\xspace}
\newcommand{\covlow}{\ensuremath{\boldsymbol{\Sigma}^{\ell}}\xspace}

\newcommand{\lmatdim}{\ensuremath{\{0,1\}^{\ell \times h}}\xspace}
\newcommand{\measurehigh}{\ensuremath{\chi^{h}}\xspace}
\newcommand{\measurelow}{\ensuremath{\chi^{\ell}}\xspace}

\newcommand{\redmatdim}{\ensuremath{\reall^{r_\ell \times r_h}}\xspace}

\newcommand{\rmatdim}{\ensuremath{\reall^{\ell \times h}}\xspace}

\newcommand{\scmhigh}{\ensuremath{\mathsf{M}^h}\xspace}
\newcommand{\scmlow}{\ensuremath{\mathsf{M}^\ell}\xspace}

\newcommand{\CAN}{\ensuremath{\mathbb{G}}\xspace}
\usepackage{spconf,graphicx}

\title{Learning Consistent Causal Abstraction Networks}
\name{Gabriele D'Acunto$^{1,2}$, Paolo Di Lorenzo$^{1,2}$, Sergio Barbarossa$^{1}$}
\address{
$^1$Information Engineering, Electronics and Telecommunications Dept.,
Sapienza University,
Rome, Italy\\
$^2$National Inter-University Consortium for Telecommunications (CNIT), Parma, Italy\\
{E-mail: \{gabriele.dacunto, paolo.dilorenzo, sergio.barbarossa\}@uniroma1.it}
}

\begin{document}
\ninept
\maketitle

\begin{abstract}
    Causal artificial intelligence aims to enhance explainability, trustworthiness, and robustness in AI by leveraging structural causal models (SCMs).  
    In this pursuit, recent advances formalize network sheaves and cosheaves of causal knowledge.
    Pushing in the same direction, we tackle the learning of consistent \emph{causal abstraction network} (CAN), a sheaf-theoretic framework where
    \emph{(i)} SCMs are Gaussian, 
    \emph{(ii)} restriction maps are transposes of constructive linear causal abstractions (CAs) adhering to the semantic embedding principle, 
    and \emph{(iii)} edge stalks correspond--up to permutation--to the node stalks of more detailed SCMs.
    Our problem formulation separates into edge-specific local Riemannian problems and avoids nonconvex objectives.
    We propose an efficient search procedure, solving the local problems with \spectral, our iterative method with closed-form updates and suitable for positive definite and semidefinite covariance matrices.  
    Experiments on synthetic data show competitive performance in the CA learning task, and successful recovery of diverse CAN structures.
\end{abstract}
\begin{keywords}
    Causal abstraction, causal artificial intelligence, network sheaves, Stiefel manifold, structural causal models.
\end{keywords}

\section{Introduction}\label{sec:introduction}
\emph{Causal artificial intelligence} \cite{b2025causal} is an emerging paradigm in AI aiming at explainability, trustworthiness, and robustness to domain and distribution shifts.
At its core lies the \emph{structural causal model} (SCM, \cite{pearl2009causality}), in which causal variables are described in terms of \emph{causal mechanisms}, namely the incoming causal relations from other variables.
By enabling both \emph{interventions} and \emph{counterfactual reasoning}, the SCM supports the three layers of the \emph{Pearl Causal Hierarchy} \cite{pearl2018book}--\emph{seeing}, \emph{doing}, and \emph{imagining}--thus going beyond purely predictive and task-specific AI.
Accordingly, an SCM induces a set of \emph{observational}, \emph{interventional}, and \emph{counterfactual} probability measures on the causal variables, hereafter referred to as \emph{causal knowledge} (CK).

In this pursuit, and taking into account the \emph{collaborative} dimension central to Agentic AI \cite{10849561}, D'Acunto and Battiloro \cite{d2025relativity} postulated \emph{\enquote{SCMs as subjective and imperfect representations of the world, that cannot be severed from the network of relations they are immersed in}}.
This generalizes the canonical setting, where subjects--e.g., AI agents, individuals, satellites--are assumed to be in bijection with, and have direct access to, a ground-truth SCM.
Additionally, the SCM is inherently tied to a \emph{perspective}, determined by its position within the network of relations.
Subjective CK spreads across the network via \emph{interventionally consistent causal abstraction} (CA) \cite{beckers2019abstracting,rischel2020category} maps and their duals, ultimately leading to the definition of relative CK: 
in words, \emph{\enquote{the CK of subject A from the perspective of subject B}}.

The construction described above is formalized via \emph{functorial} constructions, called \emph{network sheaf and cosheaves of CK} \cite{d2025relativity}.
Here, CK--represented as convex spaces of probability measures \cite{fritz2009convex}--is attached to nodes and edges of the network, and affine measurable maps correspond to node-edge incidence relations.
The former are called node and edge stalks, the latter are restriction (for sheaves) and extension (for cosheaves) maps.
This novel functorial framework potentially opens to AI systems where \emph{global CK} emerges \emph{inductively} from the interaction of causal AI agents, while local CK at the agent level can, in turn, be \emph{deductively} shaped by the system.
Our work moves in this direction, focusing on learning a specific instance of this framework, the \emph{causal abstraction network} (CAN, \cite{d2025cantl}).

\spara{Related works.}
We work under the setting of \cite{d2025causal}, specifically without assuming: 
\emph{(i)} complete specification of the SCMs;
\emph{(ii)} knowledge of their causal graphs;
\emph{(iii)} their functional forms;
\emph{(iv)} availability of interventional data; or
\emph{(v)} jointly sampled observational data.
As in \cite{d2025relativity}, we consider CAs as mappings between SCMs and adopt the \abst-abstraction framework \cite{rischel2020category}.
In detail, we focus on \emph{constructive linear causal abstractions} (CLCAs, \cite{massidda2024learning,beckers2019abstracting}) adhering to the \emph{semantic embedding principle} (SEP, \cite{d2025causal}), informally stating that the high-level CK must be preserved when embedded in the low-level.
Our problem formulation builds upon \Cref{th:existenceCA} in \cite{d2025causal}, which focuses on CLCA learning outside the sheaf-theoretic framework.
Finally, \cite{d2025cantl} represents an extended version of this paper,  introducing the CAN framework and studying its theoretical properties.

\spara{Contributions.}
Given a collection of Gaussian measures from subjective SCMs, we formulate the problem of learning a consistent CAN.
Our formulation separates into edge-specific local problems corresponding to CLCA learning tasks.
For these, we propose an alternative to \cite{d2025causal}, which avoids nonconvex objectives and explicitly characterizes the set of optimal CLCAs. 
Our solution is an efficient search procedure based on \Cref{th:existenceCA} and the compositionality of CLCAs.
This procedure employs the Riemannian method \spectral to solve the local problems iteratively with updates in closed form.
Unlike the methods in \cite{d2025causal}, \spectral applies to both positive definite and positive semidefinite covariance matrices. 
The latter are especially relevant for global sections, consistent assignments of probability measures---either observational or interventional---to each node of the network that does not break local rules \cite{d2025relativity}. 
Experiments on synthetic data show that \spectral performs comparably to the methods in \cite{d2025causal} when learning the restriction maps from positive definite covariances.
Furthermore, our search procedure successfully recovers different CAN structures when global sections are observed, i.e., when positive semidefinite covariances are provided as inputs.

\smallskip
\emph{Our work is a step toward practical learning of causal abstraction networks, with potential for agentic and collaborative AI systems.}

\section{A Primer on Causal Abstraction Networks}\label{sec:background}

This section gives the notation and key concepts of consistent CANs.
For a comprehensive treatment please refer to \cite{d2025cantl}.

\spara{Notation.} The set of integers from $1$ to $n$ is $[n]$.
The vector of zeros of size $n$ is $\zeros_n$.
The identity matrix of size $n \times n$ is $\identity_n$. 
The set of positive definite matrices over $\reall^{n\times n}$ is $\pd^n$.
The Stiefel manifold is $\stiefel{\ell}{h} \!\coloneqq\! \{ \V \!\in \reall^{\ell \times h} \mid \V^\top\V \!= \identity_h \}\,.$
The Hadamard product is $\odot$.
$\underline{\A}$ is the column-wise vectorization of \A.

\spara{Consistent causal abstraction network.}
Let us consider a collection of $N$ subjective SCMs, denoted by $\scmcollection$, modeling the same phenomenon at different levels of granularity.
For instance, different teams of neuroscientists using different brain atlases to study effective connectivity \cite{friston2011functional}.
Each SCM $\scm{\ell}$ has a different number of $\ell$ causal variables $\myendogenous^\ell=\{X_1,\ldots, X_\ell\}$--the regions of interest (ROIs) individuated by the atlas--and we assume it induces an observational zero-mean Gaussian measure $\measure^\ell \sim N(\zeros_\ell, \covlow)$--the joint probability density of the signals from ROIs.
Thus, we say the SCM is Gaussian.

We assume the SCMs relate through CLCAs.
Specifically, consider two SCMs \scm{\ell} and \scm{h}, where the former is the low-level (fine-grained) and the latter is the high-level (coarse-grained abstraction).
A CLCA is a pair denoted by $\langle \B_{h\ell}, \V_{h\ell} \rangle$, where 
\emph{(i)} $\B_{h\ell} \in \{0,1\}^{h \times \ell}$ encodes a structural disjoint partitioning of the low-level variables $\myendogenous^\ell$ into the high-level $\myendogenous^h$,
and \emph{(ii)} $\V_{h\ell} \in \reall^{h \times \ell}$ is a linear mapping from the outcome of low-level variables to that of the high-level, whose nonzero entries are determined by ones in $\B_{h\ell}$.
Here we focus on particularly well-behaved CAs, i.e., those adhering to SEP.
The SEP implies that going from the high-level model \scmhigh to the low-level model \scmlow and then abstracting back to \scmhigh allows for perfect reconstruction of CK.
For CLCAs, SEP relates to the geometry of the Stiefel manifold, implying $\V_{\ell h} \in \stiefel{\ell}{h}$, where $\V_{\ell h}=\V_{h \ell}^\top$. 
Additionally, it induces a necessary condition for Gaussian measures, an application of the Ostrowski's theorem for rectangular matrices \cite{higham1998modifying}.
\begin{theorem}[From \cite{d2025causal}]\label{th:existenceCA}
    Let $\measurelow \sim N(\zeros_\ell, \covlow)$, $\measurehigh \sim N(\zeros_h, \covhigh)$, where $\covlow \in \pd^{\ell}$ and $\covhigh \in \pd^{h}$.
    Denote by $0<\lambda_1\leq \ldots \leq \lambda_{\ell}$ the eigenvalues of \covlow, and by $0<\kappa_1 \leq \ldots\leq \kappa_{h}$ those of \covhigh.
    If a CLCA \V complying with SEP from \measurelow to \measurehigh exists, then
    \begin{equation}\label{eq:spectralCA}
        \lambda_i \leq \kappa_i \leq \lambda_{i + \ell -h}, \quad \forall \,i \in [h]\,.
    \end{equation}
\end{theorem}
From \scmcollection, we build a poset defined by $\ell \leqslant h$ iff \scm{h} is a CLCA of \scm{\ell} adhering to SEP.
In light of CA compositionality \cite{d2025cantl}, for parsimony, we consider the \emph{transitive reduction} of the poset, keeping only \emph{irreducible} relations, i.e., $\ell \leqslant h$ such that no $m$ exists with $\ell \leqslant m \leqslant h$.
The dynamics of CK is modelled through a consistent CAN, capturing both abstraction and embedding.
\begin{definition}[Consistent Causal Abstraction Network \cite{d2025cantl}]\label{def:CAN}
    A consistent CAN \CAN is an oriented undirected graph with:
    \emph{(i)} nodes corresponding to convex spaces of probability measures induced by Gaussian SCMs $\scm{\ell} \in \scmcollection$; 
    and \emph{(ii)} for each edge $e_{\ell h}$ determined by $\ell \leqslant h$, a linear embedding map $\V_{\ell h} \in \stiefel{\ell}{h}$ when moving from \scm{h} to \scm{\ell}, and a CLCA $\V_{h \ell}=\V_{\ell h}^\top$ when abstracting from \scm{\ell} to \scm{h}.
    The orientation of \CAN follows the embedding direction.
\end{definition}
The CAN is said consistent since $\V_{\ell h} \in \stiefel{\ell}{h}$ implies commutation of the composition of the restriction maps along every oriented loop $\scm{h}\stackrel{\text{embedding}}{\rightarrow}\scm{\ell}\stackrel{\text{abstraction}}{\rightarrow}\scm{h}$ \cite{d2025cantl}.
Together with a non-empty kernel of the associated combinatorial Laplacian, consistency is a necessary condition for the existence of global sections \cite{d2025cantl}.
This result has immediate implications for the statistical properties of global sections, which are consistent assignments of probability measures $\chi\coloneqq\{\measure_1, \ldots, \measure_N\}$ to each node of the network that does not break local rules \cite{d2025relativity,d2025cantl}.  
\begin{corollary}\label{cor:same_spectrum}
    Let $\chi$ be a global section.
    Then, the covariance matrices of its elements $\chi_i\sim N(\zeros_{d_i}, \boldsymbol{\Sigma}^{d_i})$, $i \in [N]$, have the same $k\leq \min_i d_i=h$ nonzero eigenvalues.  
\end{corollary}
The proof is omitted for lack of space, but is available in \cite{d2025cantl}.
Similarly to graph signal processing \cite{ortega2018graph}, \emph{smoothness} w.r.t. a consistent \CAN is defined.
Intuitively, a collection of SCMs is smooth w.r.t. \CAN if, for every edge $e_{\ell h}: \ell \sim h$, the coarser \scm{h} is---approximately---a CLCA of the finer \scm{\ell}.
Following \cite{d2025causal}, smoothness can be measured in terms of how well the abstraction of the low-level probability measure \measurelow approximates the high-level \measurehigh.
Specifically, denoting by $\KL{\alpha}{\mu}{\nu}$ the Kullback--Leibler (KL) divergence between $\mu$ and the pushforward measure induced by the map $\alpha$ applied to $\nu$, we measure smoothnes as the sum over $e_{\ell h} \in \edgeset$ of $\KL{\V_{h \ell}}{\measurehigh}{\measurelow}$.
\section{Learning causal abstraction networks}\label{sec:learning_CAN}

Our input is a collection $\chi=\{\chi_1, \ldots, \chi_N\}$ of Gaussian probability measures--sorted in descending order of dimensionality--induced by $N$ SCMs, assumed to be smooth w.r.t. \CAN.  
From \Cref{def:CAN}, learning a consistent CAN amounts to learning \emph{(i)} the CA relations determining its edge set \edgeset, 
and \emph{(ii)} the CLCAs $\{\V_{\ell h}^\top\}$ for each $e_{\ell h} \in \edgeset$, where $\V_{\ell h} \in \stiefel{\ell}{h}$.
Let $\mathcal{F}$ be the set of edges satisfying the CLCA necessary condition in \cref{th:existenceCA}.
Then, our learning problem is
\begin{equation}\label{eq:globalProblem}
    \begin{aligned}
        \min_{\{\V_{\ell h} \in \stiefel{\ell}{h}\}} \; \sum_{e_{\ell h} \in \mathcal{F}} \KL{\V_{\ell h}^\top}{\measurehigh}{\measurelow}\,. 
    \end{aligned}
    \tag{P1}
\end{equation}

Notably, the set of CLCAs $\V_{\ell h}^\top$ for which the local KL vanishes identifies the learned edge set of \CAN.
By construction $\V_{\ell h}$ only makes sense in relation to a CA, thus the solution to \eqref{eq:globalProblem} is the latter set of CLCAs.
Problem \eqref{eq:globalProblem} is separable along the edges, and each local problem can be solved with the algorithms proposed by \cite{d2025causal} when both \covlow and \covhigh are positive definite.
However, as global sections with positive semidefinite covariances (cf. \Cref{cor:same_spectrum}) must also be accounted for, existing methods must be generalized.
Moreover, even in the positive definite case \eqref{eq:globalProblem} remains challenging for several reasons.
First, the search over $\mathcal{F}$ requires solving several CA learning problems, at most $N(N-1)/2$. 
This becomes computationally expensive as $N$ increases, also because each local problem is nonconvex and must be solved multiple times with different random initializations to obtain reliable results \cite{d2025causal}.
Second, considering zero-mean Gaussians and denoting by $r_h$ the rank of \covhigh, the KL divergence
\begin{equation}\label{eq:KL}
    \begin{aligned}
        &\KL{\V_{\ell h}^\top}{\measurehigh}{\measurelow}=\Tr{\left( \V_{\ell h}^\top \covlow \V_{\ell h}\right)^\dag \covhigh} +\\
        &+ \log \mathrm{gdet} \left\lbrace\V_{\ell h}^\top \covlow \V_{\ell h} \right\rbrace - \log \mathrm{gdet} \left\lbrace \covhigh \right\rbrace - r_h\,,
    \end{aligned}
\end{equation}
is nonconvex in $\V_{\ell h}$ and costly to optimize \cite{d2025causal}.
Third, the Stiefel manifold imposes nonconvex orthogonality constraints on $\V_{\ell h}$.

To address these challenges, we propose an efficient search strategy that leverages the compositionality of CAs \cite{d2025cantl}.
In addition, building on \Cref{th:existenceCA}, we introduce a problem formulation alternative to \cite{d2025causal} for local problems avoiding nonconvex, costly objectives, and the related \spectral algorithm that applies to both positive definite and semidefinite cases.
Throughout, we assume full prior knowledge of the CA structure $\B_{h\ell}$, which is still a challenging setting \cite{d2025causal}.

\spara{Search procedure.}
Consider $\chi$ as above.
For each $\boldsymbol{\Sigma}^{d_i}$, $i \in [N]$, compute the eigenvalues.  
We encode possible CA relations according to \Cref{th:existenceCA} into a strictly lower triangular binary matrix $\mathbf{P} \in \{0,1\}^{N \times N}$.
Rather then applying \Cref{th:existenceCA} $N(N-1)/2$ times we fill $\mathbf{P}$ by leveraging the compositionality of CAs.
Specifically, we first apply \Cref{th:existenceCA} to the $N-1$ SCM pairs on the first subdiagonal, i.e., to consecutive neighbors in the chain $\chi_1-\chi_2-\ldots-\chi_N$.
We then set to $1$ the entries of $\mathbf{P}$ corresponding to pairs satisfying \Cref{th:existenceCA}, and leaving the others null.
Subsequently, we set $\mathbf{P}$ equal to its transitive closure.
Next, we move to the second subdiagonal.
At this stage, it is necessary to test at most $N-2$ SCM pairs, since some entries may already have been filled through the transitive closure at the previous step.
After the update, we apply again the transitive closure.
We repeat this procedure for all remaining subdiagonals.

Once $\mathbf{P}$ is constructed, we solve the local problems corresponding to its non-null entries.
Here too, the number of subproblems to be solved can be reduced by exploiting transitive closure.
Consider an initial empty $N \times N$ binary matrix $\mathbf{A}$.
We first solve the local problems for the nonnull entries of the first subdiagonal of $\mathbf{P}$, and set to $1$ the entries of $\mathbf{A}$ whenever the CLCA holds, i.e., when the stopping criteria for the local problem are matched.
We then compute the transitive closure of $\mathbf{A}$ and remove from $\mathbf{P}$ those entries already implied by this closure.
The same procedure is applied iteratively to the subsequent subdiagonals.
By construction, the learned $\mathbf{A}$ is the transitive reduction of all possible CA relations implied by the compositionality property.

\spara{Local problem.} 
We now turn to the alternative formulation of the CA learning local problem for a single edge $e_{\ell h}$.
Since $e_{\ell h} \in \mathcal{F}$, thus \Cref{th:existenceCA} holds and a CLCA between \scm{\ell} and \scm{h} might exist.
Let $\V_{\ell h}=\B \odot \V$, where $\B = \B_{h\ell}^\top$.
Since we consider zero-mean Gaussians, the inputs are $\covlow$ and $\covhigh$, generally positive semidefinite with ranks $r_\ell$ and $r_h$, respectively.
If the CLCA exists, then any $\V \in \stiefel{\ell}{h}$ such that 
\begin{equation}\label{eq:optabst}
    (\B \odot \V)^\top\covlow (\B \odot \V)=\covhigh\,,
\end{equation}
is optimal for the edge $e_{\ell h}$, i.e., it makes the local KL in \Cref{eq:KL} null.
Consider the eigendecompositions
\begin{equation}\label{eq:eigd}
    \covlow=\U_{\ell,+} \Lambda_{\ell,+} \U_{\ell,+}^\top\, , \text{ and } \covhigh=\U_{h,+} \Lambda_{h,+} \U_{h,+}^\top\,.
\end{equation}

Further, define the matrices $\A = \U_{\ell,+}  \Lambda_{\ell,+}^{\frac{1}{2}} \in \reall^{\ell \times r_\ell}$ and $\C=\U_{h,+}  \Lambda_{h,+}^{\frac{1}{2}} \in \reall^{h \times r_h}$.
Substituting \Cref{eq:eigd} into \Cref{eq:optabst}, and multiplying both sides on the left by $\C^\top$ and on the right by \C yields the condition
\begin{equation}\label{eq:spectral_cond}
    \T^\top \T = \identity_{r_h}\,, \text{ with } \T=\A^\top (\B \odot \V) \C\,.
\end{equation}

Thus, any $\V \in \stiefel{\ell}{h}$ consistent with the structure $\B$ and satisfying \Cref{eq:spectral_cond} yields an optimal CLCA.
Additionally, from the expression of \T, it is immediate to see that \Cref{eq:spectral_cond} holds whenever the principal axes of \covlow, projected via the CLCA $(\B \odot \V)^\top$, align with those of \covhigh.
Hence, we tackle the following feasibility problem:

\begin{equation}\label{eq:learning_problem}
    \begin{aligned}
        \text{find} \quad &\V, \T \in \stiefel{\ell}{h} \times \stiefel{r_\ell}{r_h} \\
        \text{such that} \quad & \T - \A^\top (\B \odot \V) \C=\zeros_{r_\ell \times r_h}\,.
    \end{aligned}
    \tag{SP1}
\end{equation}

\spara{The \spectral method.}
Handling the equality constraint in \eqref{eq:learning_problem} is challenging because $\V$ and $\T$ lie on Stiefel manifolds of different dimensionality, which are nonconvex.  
To address this, we exploit the rationale underlying the \emph{splitting of orthogonality constraints} (SOC, \cite{lai2014splitting}). 
Specifically, we relax $\V$ to the Euclidean space $\rmatdim$ and enforce its orthogonality through a splitting variable $\Y \in \stiefel{\ell}{h}$, subject to 
\begin{equation}\label{eq:constraintV}
    \Y - \B \odot \V = \zeros_{\ell \times h}\,. 
\end{equation}
Combining \Cref{eq:spectral_cond,eq:constraintV} and introducing scaled dual variables $\BPsi \in \rmatdim$ and $\BUpsilon \in \redmatdim$, we obtain the augmented Lagrangian of \eqref{eq:learning_problem},
\begin{equation}\label{eq:aL}
    \begin{aligned}
        \mathcal{L}(\V, \Y, \T, \BUpsilon, \BPsi) &= \frac{1}{2}\frob{\B \odot \V - \Y + \BPsi}^2 +\\
        &+ \frac{1}{2}\frob{\T - \A^\top (\B \odot \V) \C + \BUpsilon}^2\,.
    \end{aligned}
\end{equation} 
We then minimize \Cref{eq:aL} iteratively with respect to the primal variables and maximize it with respect to the dual variables through ADMM \cite{boyd2011distributed}.
Notably, the primal variables has updates in closed form.
The duals are the canonical running sum of primal residuals.

\epara{Update for \V.}
Consider the iteration $k$.
The subproblem to solve is
\begin{equation}
    \begin{aligned}
        \V^{k+1} &= \argmin_{\V} \frac{1}{2}\frob{\B \odot \V - \Y^k + \BPsi^k}^2 + \\
        &+ \frac{1}{2}\frob{\T^k - \A^\top (\B \odot \V) \C + \BUpsilon^k}^2\,.
    \end{aligned}
\end{equation}
The solution is given in closed, vectorized form.

\begin{lemma}\label{lem:updateV}
    Let $\mathcal{B}\coloneq\{i \in [\ell h] \mid b_i=1, b_i \in \underline{B}\}$, $\K = \C \otimes \A$, and $\mathbf{b}= \underline{\Y}^k - \underline{\BPsi}^k +\underline{\A \T^k \C^\top} -\underline{\A \BUpsilon^k \C^\top}$.
    Then,
    \begin{equation}\label{eq:updateV}
        \underline{\V}_{\mathcal{B}} = \left(\identity_{|\mathcal{B}|}+\K_\mathcal{B} \K_\mathcal{B}^\top\right)^{-1} \mathbf{b}_{\mathcal{B}}\,.
    \end{equation}
\end{lemma}
\begin{proof}
    Differentiating \Cref{eq:aL} with respect to $\V$ and vectorizing via the Kronecker product yields the stationarity condition
    \begin{equation}\label{eq:stationarityV}
        \zeros_{\ell h} = \left(\D + \D \K (\D \K)^\top\right) \underline{\V} - \D \mathbf{b}\,,
    \end{equation}
    where $\D = \mathrm{diag}(\underline{\B})$.
    Only entries with $\D_{ii}=1$ are active in $\underline{\V}$; the others are zero. 
    Restricting to indices in $\mathcal{B}$ gives \Cref{eq:updateV}.
\end{proof}

\epara{Update for \Y.}
The subproblem to solve is
\begin{equation}\label{eq:updateY}
    \begin{aligned}
        \Y^{k+1} &= \argmin_{\Y} \frac{1}{2}\frob{\B \odot \V^{k+1} - \Y + \BPsi^k}^2 =\\
        &= \prox_{\stiefel{\ell}{h}}(\B \odot \V^{k+1} + \BPsi^k)\,;
    \end{aligned}
\end{equation}
where $\prox_{\stiefel{\ell}{h}}(\mathbf{S})$ is given by the $\U$ factor of the polar decomposition of $\mathbf{S}$ \cite{d2025causal}.

\epara{Update for \T.}
Analogously to \Y, we obtain
\begin{equation}\label{eq:updateT}
    \begin{aligned}
        \T^{k+1} &= \argmin_{\T} \frac{1}{2}\frob{\T - \A^\top (\B \odot \V^{k+1}) \C + \BUpsilon^k}^2 =\\
        &= \prox_{\stiefel{r_\ell}{r_h}}(\A^\top (\B \odot \V^{k+1}) \C + \BUpsilon^k)\,.
    \end{aligned}
\end{equation}

Following \cite{boyd2011distributed}, empirical convergence is assessed via primal and dual residuals, which must vanish as $k \to \infty$.
See \cite{d2025cantl} for details.

\spara{Computational cost comparison.}
The key advantage of \spectral lies in providing closed-form updates. 
Asymptotically, the most expensive step per iteration is the update of \V in \Cref{eq:updateV}, which costs $\mathcal{O}(\ell^3)$ flops. 
In contrast, the methods in \cite{d2025causal} incur a cost of $\mathcal{O}(K \ell^3)$, where $K$ is the number of their corresponding inner iterations–-each costing $\mathcal{O}(\ell^3)$ asymptotically–-required to solve the \V-subproblem. 
\section{Empirical assessment}\label{sec:empirical_assessment}

To empirically evaluate the performance of the proposed learning method, we first test the local solver \spectral, and then address the learning of different CAN structures.

\spara{Performance on the local problem.}
Given as input the low- and high-level models, viz. \scm{\ell} and \scm{h}, the goal is to recover the CLCA matrix $\V^\top$ under full prior knowledge of \B.  
We compare \spectral with the CLinSEPAL, LinSEPAL-ADMM, and LinSEPAL-PG methods recently proposed in \cite{d2025causal}.  
These baselines minimize the KL divergence in \Cref{eq:KL} for the case of positive definite matrices, relying on different optimization frameworks (cf. \cite{d2025causal}).  

We tested the methods on the same synthetic dataset used in the full prior setting of \cite{d2025causal}.  
The dataset includes three configurations $(\ell, h) \in \{(12,2), (12,4), (12,6)\}$, corresponding to \emph{high}, \emph{medium-high}, and \emph{medium} abstraction levels.  
Please refer to \cite{d2025causal} for details on the data generating process.
To evaluate performance, we monitor 
\emph{(i)} constructiveness, i.e., the fraction of learned CAs that are constructive, 
\emph{(ii)} $\KL{\Vhat^\top}{\measurehigh}{\measurelow}$ for quantifying the alignment, 
\emph{(iii)} normalized absolute Frobenius distance between the estimate \Vhat and the ground truth $\V^\star$, 
and \emph{(iv)} F1 score for structural interventional consistency.

\Cref{fig:full_synth_data} shows the performance--we omit constructiveness and F1 score for space reasons--of the tested methods.
Vertical bars are interquartile ranges across $30$ simulations. 
All methods consistently yield CLCAs for all simulations and show F1 score equal to 1, thus recovering the ground truth structure encoded by $\B$.
The Frobenius distance shows that, under high coarse-graining, identifying $\V^\star$--up to sign--is more challenging, as the set of $\V$ satisfying \eqref{eq:learning_problem} enlarges.  
Finally, the methods also achieve good alignment.

To sum up, these results show that \spectral achieves comparable performance to the methods in \cite{d2025causal} in learning CLCAs with positive definite covariances as inputs, while avoiding nonconvex, costly objectives and providing updates in closed form.  

\begin{figure}[t]
    \centering
    \includegraphics[width=\columnwidth]{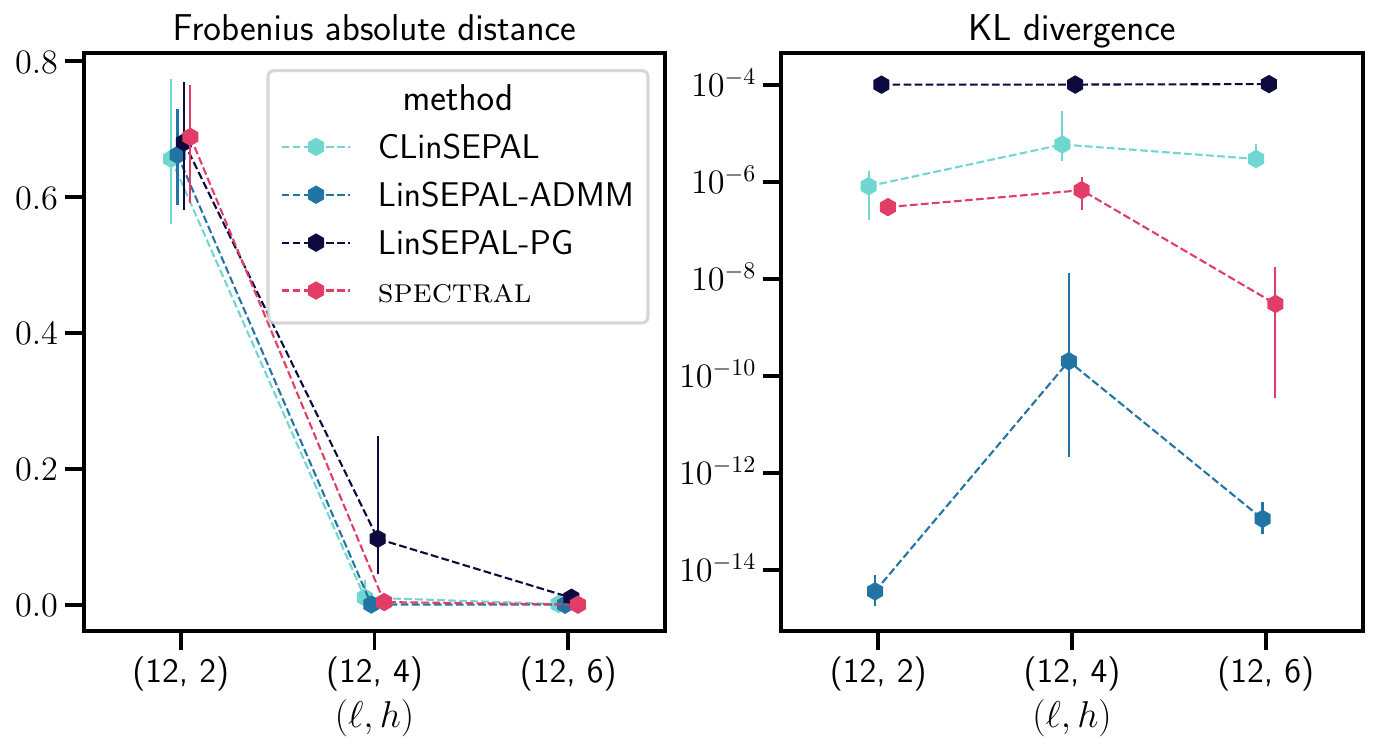}
    \caption{Synthetic results for the solution of the local problem across all settings $(\ell, h)$ from \cite{d2025causal}. 
    }
    \label{fig:full_synth_data}
\end{figure}

\spara{Performance on CAN learning.}
We then turn to the learning of CANs.
We set $N=10$ and assign random dimensions $\ell \in [2, 2N]$ to each \scm{\ell}.
We then attach the SCMs to the nodes of three graphical structures:
\emph{(i)} the \emph{chain} whose full CAN (transitive closure) is the fully-connected graph;
\emph{(ii)} the \emph{star}, which coincides with the full CAN;
\emph{(iii)} the \emph{tree} whose full CAN results in a denser tree.
For each CLCA $\ell \leqslant h$ in the transitive closure, we randomly generate $\B_{\ell h} \in \lmatdim$ and sample $\V_{\ell h} \in \stiefel{\ell}{h}$ accordingly. 

All considered structures admit global sections, which can be generated iteratively starting from the coarsest \measurehigh (cf. \cite{d2025cantl}).  
For each structure, we generate $S=30$ global sections and run the search procedure described in \Cref{sec:learning_CAN}, using \spectral to solve each local problem.  
Because of nonconvexity, each local problem is solved by running \spectral up to $\ntrials \in \{10,100\}$ times, with different initializations.  
If residual convergence is established within $\ntrials$ runs, we stop and set the corresponding entry of \A to $1$.  
In the experiments, we set residuals tolerances to $10^{-3}$ and cap the maximum number of iterations at $K=1000$.  
The performance is evaluated in terms of \emph{false positive rate} (FPR) and \emph{true positive rate} (TPR) w.r.t. the full CAN.  
By \Cref{cor:same_spectrum}, the covariances of nodes $\measure_1, \ldots, \measure_9$ are positive semidefinite and share the same nonzero eigenvalues, corresponding to the spectrum of $\measure_{10}$. 
Hence, only \spectral can solve the local CA learning problems.
This is an extreme scenario as well, since the necessary condition of \Cref{th:existenceCA} holds true for all $N(N-1)/2$ SCM pairs.  
Consequently, the matrix $\mathbf{P}$ has all entries equal to $1$ in its strictly lower triangular part, and thus provides no information to reduce the number of local CA learning problems.  
\begin{figure}[t]
    \centering
    \includegraphics[width=\columnwidth]{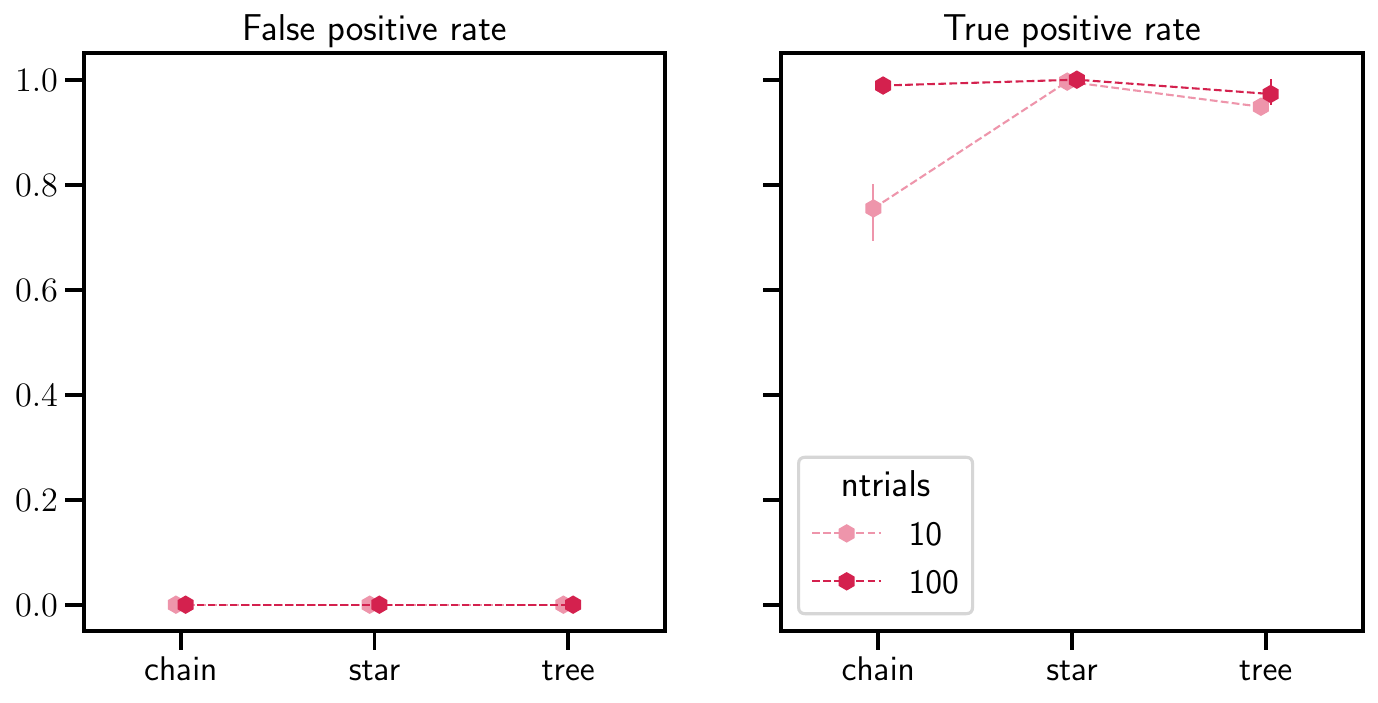}
    \caption{False positive (left) and true positive (right) rates for the proposed search procedure on the three CAN structures.}
\label{fig:CAN_learning}
\end{figure}
\Cref{fig:CAN_learning} shows the results.
Different values of \ntrials are given in hues of red, while vertical bars are interquartile ranges across the $S$ simulations.
Although \Cref{th:existenceCA} holds true, the FPR (left) confirms that \spectral never identifies a CLCA when none exists in the ground truth.
This holds across all settings and values of \ntrials.
Conversely, the TPR (right) depends on \ntrials.
As expected, due to nonconvexity, the initialization is important since residual convergence may not be reached within the maximum $K$ iterations.  
Overall, the results confirm that our search procedure equipped with \spectral achieves nearly optimal TPR across all the settings when $\ntrials=100$.   
\section{Conclusions and future works}\label{sec:conclusions}
This paper moves a step forward a principled foundation for collaborative causal AI systems. 
We tackled the learning of consistent CAN, a sheaf-theoretic framework where subjective Gaussian SCMs align and glue together via transposes of CLCAs.
Our problem formulation avoids the optimization of nonconvex objectives and its solution is an efficient search procedure using the local Riemannian method \spectral with closed-form updates.  
Unlike prior approaches, \spectral applies to both positive definite and semidefinite covariance matrices, enabling the recovery of CAN structures even when we observe global sections.  
Experiments on synthetic data confirmed the effectiveness of our approach, matching the performance of existing methods in CA learning and successfully retrieving CAN structures. 

Many theoretical results are omitted for lack of space, but they are available in the extended version of this paper \cite{d2025cantl}.
Notably, they highlight the impact of topology, which in some cases prevents the existence of global sections. 
Developing methodologies that, for a given topology, recover the sparsest structure inducing the same set of CAs--and minimally modify it to support global sections--is an important future step.
Another important aspect concerns the identifiability of the ground truth CLCA.
Working under the general setting of \cite{d2025causal}, we can only guarantee structural interventional consistency of the learned maps through structural priors.
Going forward by ensuring counterfactual consistency \cite{xia2024neural} is important, although it might require availability of interventional and counterfactual data.
Finally, moving beyond synthetic and hybrid settings is essential to fully evaluate CAN potential in real-world scenarios.
This is a crucial aspect for the sheaf-theoretic framework \cite{d2025relativity} in general, particularly because real-world applications are still largely lacking.
\clearpage
\spara{Acknowledgements.}
The work was supported by the SNS JU project 6G-GOALS \cite{strinati2024goal} under the EU’s Horizon program Grant Agreement No 101139232, by Huawei Technology France SASU under Grant N. Tg20250616041, and by the  European Union under the Italian National Recovery and Resilience Plan of NextGenerationEU, partnership on Telecommunications of the Future (PE00000001- program RESTART).

\balance
\bibliographystyle{IEEEbib}
\bibliography{bibliography}

\end{document}